\documentclass[conference]{IEEEtran}
\IEEEoverridecommandlockouts

\usepackage{cite}
\usepackage{amsmath,amssymb,amsfonts}
\usepackage{algorithmic}
\usepackage{graphicx}
\usepackage{textcomp}
\usepackage{xcolor}
\usepackage{multicol}
\usepackage{booktabs}
\usepackage[draft]{hyperref}
\usepackage{float}
\usepackage[ruled, vlined, noend, boxed]{algorithm2e}
\usepackage{amsthm}
\usepackage{tcolorbox}

\usepackage{mathrsfs}

\usepackage{subcaption}

\tcbset{
  finding/.style={
    colback=gray!10!white,
    colframe=black,
    boxrule=0.4pt,
    arc=3pt,
    left=4pt,
    right=4pt,
    top=4pt,
    bottom=4pt,
    fonttitle=\bfseries
  }
}

\newtheorem{theorem}{Theorem}
\newtheorem{lemma}{Lemma}

\begin{document}

\title{PPAI: Enabling Personalized LLM Agent Interoperability for Collaborative Edge Intelligence}

\author{

\IEEEauthorblockN{Zile Wang$^{1,*}$, Qianli Liu$^{1,*}$, Kaibin Guo$^{2}$, Haodong Wang$^{1}$, Jian Lin$^{1}$, Zicong Hong$^{1}$ and Song Guo$^{1}$}\IEEEauthorblockA{$^1$Department of Computer Science and Engineering, The Hong Kong University of Science and Technology, Hong Kong \\$^2$School of Software Engineering, Sun Yat-Sen University, China
\\zwangox@connect.ust.hk, qianli.liu@connect.ust.hk, guokb@mail2.sysu.edu.cn, hwanghb@connect.ust.hk, \\jlindc@connect.ust.hk, congcong@ust.hk, 
songguo@cse.ust.hk}
\thanks{
$^*$Equal contribution. 

This research was supported by fundings from the Hong Kong RGC General Research Fund (152169/22E, 152228/23E, 162161/24E, 162116/25E), Research Impact Fund (No. R5060-19, No. R5011-23), Collaborative Research Fund (No. C1042-23GF), NSFC/RGC Collaborative Research Scheme (Grant No. 62461160332 \& CRS\_HKUST602/24), Areas of Excellence Scheme (AoE/E-601/22-R), and the InnoHK (HKGAI). Corresponding authors: Zicong Hong, Song Guo. 

}
}

\maketitle

\begin{abstract}
Deploying large language model (LLM) on edge device enables personalized LLM agents for various users. The growing availability of diverse personalized agents presents a unique opportunity for peer-to-peer (P2P) collaboration, wherein each user can delegate tasks beyond the local agent's expertise to remote agents more suited for the specific query. 
This paper introduces PPAI, the first personalized LLM agent interoperability system, which enables users to collaborate with each other based on agent specialization.
However, the ever-changing pool of agents and their interchangeable capacity introduce new challenges when it comes to matching queries to agents and balancing loads, compared with existing P2P systems.
Therefore, we propose a scalable query-agent pair scoring mechanism based on prototypes to identify suitable agents within a P2P network with churn. 
Moreover, we propose a multi-agent interoperability Bayesian game to balance local demand and global efficiency, when changes in remote agent load occur too quickly to be observed.
Finally, we implement a prototype of PPAI and demonstrate that it substantially broadens the range of tasks that could be carried out while maintaining load balance.
On average, it achieves an accuracy improvement of up to $7.96\%$ across multiple tasks, while reducing latency by $16.34\%$ compared to the baseline.
\end{abstract}



\section{Introduction }

The rapid proliferation of AI-capable PCs, fueled by advances in on-device GPUs and dedicated AI accelerators, is reshaping the computing landscape. Modern consumer-grade devices now possess the capacity to run sophisticated large language models (LLM) locally, enabling the deployment of LLM-based agents at the edge~\cite{lu2025demystifying,belcak2025small,wang2025d, fang2025klotski}. 
To enable these LLM agents to handle specialized tasks or domains, users can personalize their local agent via fine-tuning~\cite{lester2021power, hu2022lora, ouyang2022training} and retrieval-augmented generation (RAG)~\cite{lewis2020retrieval, asai2023self}. 
While both approaches create opportunities for the flourishing of diverse personalized LLM-based agents at the edge.

\begin{figure}[t]
  \centering
  \includegraphics[width=\columnwidth]{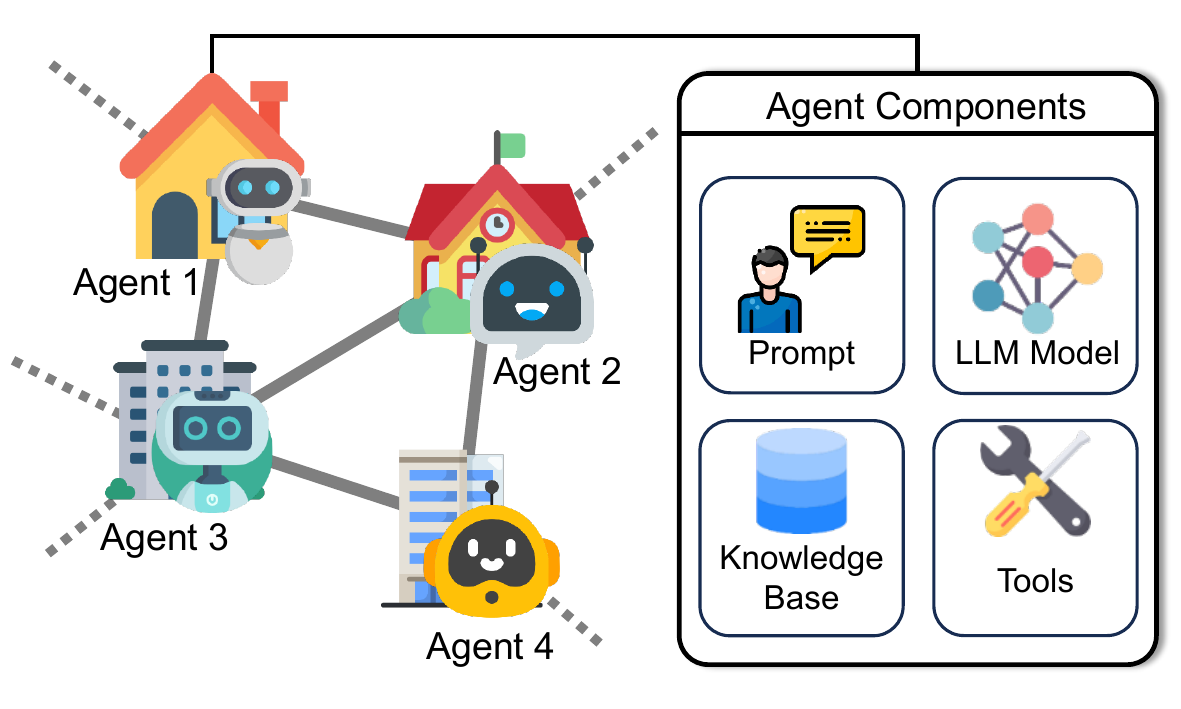}
  \caption{Our vision for PPAI: Two decades ago, P2P networks (e.g. BitTorrent) allowed users to obtain files they did not have locally by downloading pieces from other users on the network. As agents advance, we envisage a future in which individuals can use others' personalized agents to complete tasks at which their own local agents are less proficient.}
  \label{fig:network}
\end{figure}

Given this growing ecosystem of diverse personalized LLM agents, exploring peer-to-peer (P2P) collaboration emerges as an increasingly compelling avenue to unlock their collective potential. P2P paradigms have been extensively studied across a range of applications, most notably in decentralized file sharing and data distribution. Classical systems~\cite{cohen2003incentives, maymounkov2002kademlia} enable efficient dissemination of large files by leveraging swarm-based chunk exchange among peers and distributed hash table (DHT) mechanism for scalable content lookup in vast unstructured networks. Later, Gnutella-inspired protocols and gossip-based overlays~\cite{jelasity2005gossip} have demonstrated robust information propagation and fault tolerance in dynamic environments.
The principal advantage of these P2P approaches lies in their scalability and adaptability to heterogeneous nodes, enabling efficient information exchange across diverse, vast users.

P2P architecture presents an opportunity to harness complementary agent competencies across devices. This enables a paradigm that routes specific tasks suitable agents in the system, maximizing overall system utility by leveraging agent strengths. There are existing works performing agent routing for given inputs. RouterDC~\cite{chen2024routerdc} establish a router based on dual-loss design, while KABB~\cite{zhang2025kabb} achieve agent routing by building a knowledge graph. These methods, however, are not designed for P2P environments. Instead, they assume on-device co-location and focus on local model selection.

To achieve dynamic agent interoperability in a large-scale network, we propose a \textit{\underline{P}2\underline{P} \underline{A}gent \underline{I}nteroperability} (PPAI) system tailored for agent-to-agent collaborative task execution as shown in \autoref{fig:network}. For example, when a user’s local agent, specialized in legal knowledge, encounters a financial query, the user can delegate the task to another peer whose agent excels in finance.
Our system dynamically matches queries to peers with both high task-specific performance and balanced load in fully decentralized environments. By establishing a P2P collaborative system where each user only deploys one specialized agent locally, these agents coordinate to solve a wider variety of tasks, effectively leveraging the collective expertise of the network with minimized resource footprints.




However, realizing this new system faces two key challenges:
1) \textbf{Personalized Agent Diversity \& Churning}. Unlike traditional P2P networks that distribute uniform data, each agent in our system has distinct specialized competencies. Moreover, the agent pool churns as users may join, leave, or update over time. As the network scales to encompass a large number of participants with evolving agent set, determining which agent is best suited for a given query becomes increasingly complex and computationally demanding.
2) \textbf{Limited Agent Load Visibility \& Congestion}.
In highly active environments, users compete for limited agent resources. Because the system is large and query arrivals are dynamic, the real-time load of peer agents fluctuates rapidly and becomes unobservable to others. Since agents are functionally interchangeable, routing decisions made solely for local optimality could overload a subset of agents while others are underutilized. As the system scales, this lack of load visibility exacerbates congestion and degrades overall efficiency.

%

To address these challenges, we design query-agent pair scoring module for dynamic matching queries to churning agents and a Bayesian game to estimate the load at other edges. 
First, to tackle query-agent matching under churning, we establish a fixed latent space. By projecting both queries and agents to such space, the pair scoring can be done through anchor-based coordinates. 
Second, to mitigate the limited visibility of other agents’ real-time load, we incorporate a probabilistic belief distribution derived from historical interactions, enabling agents to estimate peer load and route to substitutes when necessary.
Our contributions are:

\begin{itemize}
    \item We develop a scalable query-agent pair scoring mechanism based on prototype anchors, enabling users to efficiently determine suitable agents for their specific query locally, supporting churning agent sets.
    \item We formulate the routing and resource allocation problem as a Bayesian game. We introduce Cost of Delegation to analyze local utility maximization with latency estimation based on belief distribution. We demonstrate such system evolves towards a Bayesian Nash equilibrium. 
    \item Our proposed system demonstrate evident improvement based on collaboration between small agents and improve average accuracy up to $7.96\%$ on multiple tasks while reducing processing time by $16.34\%$ than baseline method under high user demand simulation.
\end{itemize}

\section{Related Work}
\label{relate}


\subsection{LLM Agent}


LLM agents are LLMs with equipped prompts, knowledge base and tools, becoming autonomous software entities that leverage the reasoning, planning, and interaction capabilities of LLMs to perceive their environment, make decisions, and execute actions toward achieving specific goals. Unlike traditional rule-based agents, LLM agents dynamically interpret user instructions, generate plans, and interact with external tools or other agents via natural language. Recent advances have demonstrated the potential of LLM agents across diverse application scenarios, including Mind2Web~\cite{deng2023mind2web} for real-world websites tasks, SWE-agent~\cite{yang2024swe} for automated software engineering, and MetaGPT~\cite{hong2023metagpt} for multi-agent collaborative programming.
These works illustrate the versatility of LLM agents and their ability to perform complex tasks through interaction with environments or cooperation among agents.

\subsection{On-device LLM Agent}

On-device LLM agents deployed on personal devices have become increasingly popular for delivering personalized AI services directly at the edge. These agents operating independently without constant server connectivity are ideal for latency-sensitive or privacy-critical scenarios. They preserve data privacy while benefiting from personalized AI capabilities.
Recent work like Mobile-Agent-v2~\cite{wang2024mobile} proposes mobile device operation assistance. MobileSteward~\cite{liu2025mobilesteward} extends this paradigm by integrating multiple app-oriented agents coordinated by a steward agent. Meanwhile, MiniRAG~\cite{fan2025minirag} presents a lightweight RAG system designed for resource-constrained devices. These works highlight the potential of local agents to deliver efficient, context-aware, and personalized capabilities without relying on powerful cloud infrastructures.

\subsection{Optimization on Agent Efficiency}



\subsubsection{Local Efficiency Optimization}

On the system side, some efforts~\cite{huang2019gpipe,chen2018tvm, liu2025mell} improve throughput via pipeline parallelism, multi-GPU coordination and compiler-level tensor optimizations across heterogeneous hardware.
Meanwhile, model compression methods based on quantization techniques such as GPTQ~\cite{frantar2023optq}
achieves low-bit inference with minimal accuracy loss and knowledge distillation methods~\cite{gu2023minillm} compress models by teacher-student learning. These techniques are orthogonal to ours for shrinking the memory of specialized agents or accelerating inference within our P2P system.

\subsubsection{Multi-agent Collaboration}

To enable performance beyond single agent, multi-agent frameworks have emerged as a promising paradigm for combining specialized capabilities across multiple models. Agent routing approaches routes each query to the most suitable agent via contrastive training objectives~\cite{chen2024routerdc}, reward-based feedback loops~\cite{lu2023routing}, and semantic alignment through knowledge graphs~\cite{zhang2025kabb}.
Cost-aware variants~\cite{ong2024routellm} further adapt routing decisions based on predicted cost constraints.
Other efforts achieve agent ensembling by post-ranking mechanisms~\cite{jiang2023llm}, weighted probability fusion~\cite{wang2023fusing}, or parameter merging~\cite{yadav2023ties}.
However, they are not intended for P2P settings. They operate under the assumption that multiple models reside on a single device and aim to choose the most suitable local model for each input.

\section{Motivation}
\label{motivation}

In this section, we analyze the characteristics of personalized agents and present some findings that motivate our design.

\begin{figure}[t]
  \centering
  \includegraphics[width=0.45\textwidth]{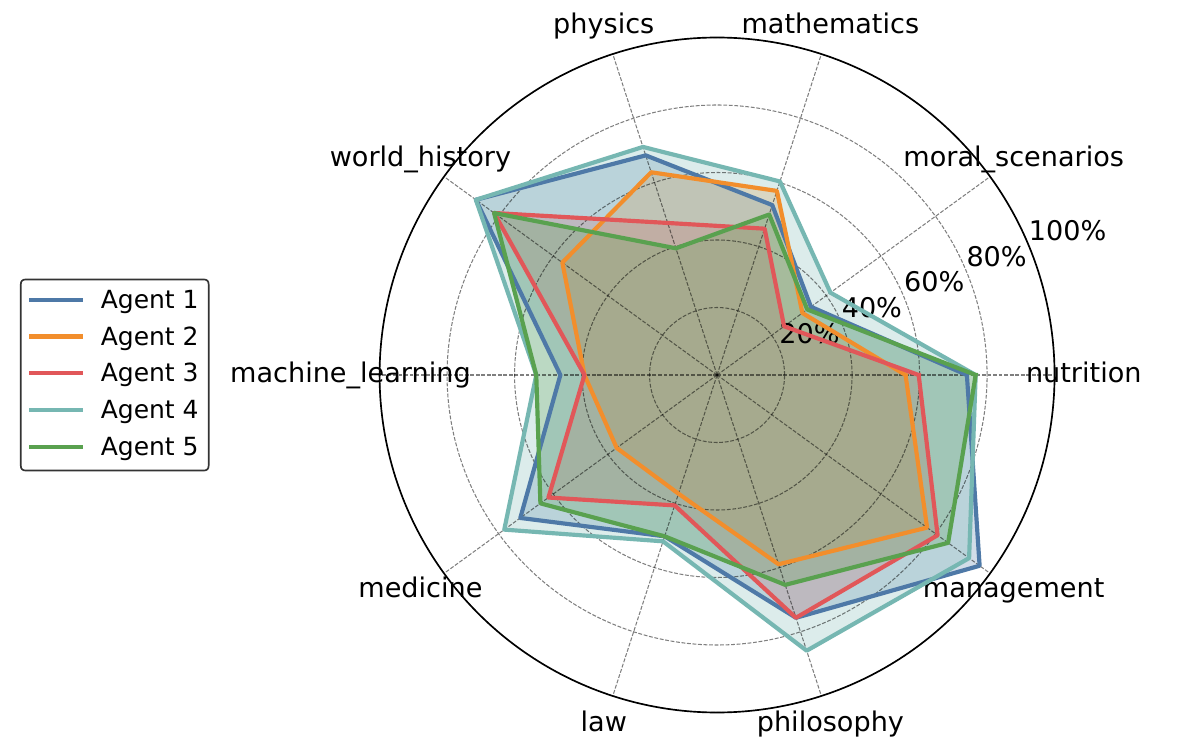}
  \caption{Comparison of model accuracies across representative MMLU tasks.}
  \label{fig:radar}
\end{figure}

\begin{figure}[t]
  \centering
  \includegraphics[width=0.75\columnwidth]{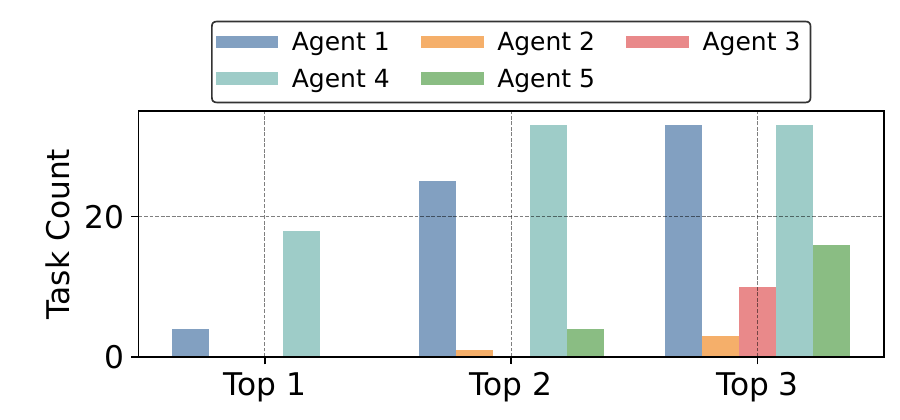}
  \caption{Task counts where each agent ranks as the top-1, top-2, or top-3 performer across MMLU tasks. 
}
  \label{fig:top}
\end{figure}

\begin{figure}[t]
  \centering
  \includegraphics[width=\columnwidth]{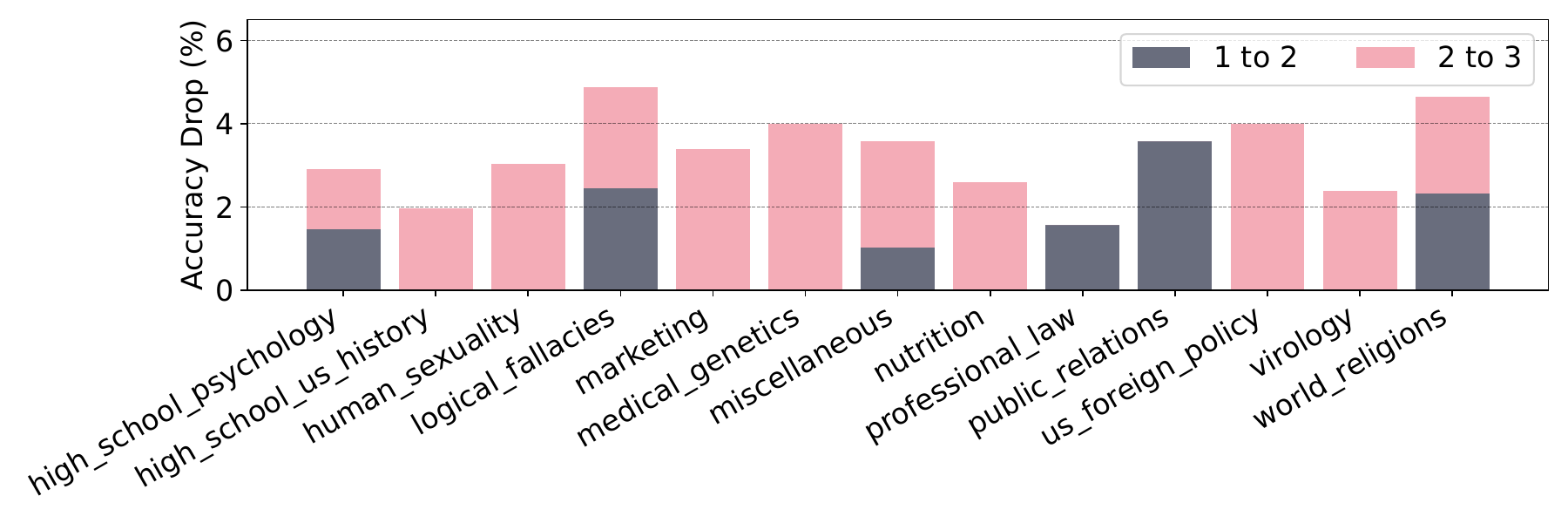}
  \caption{Accuracy degradation across some tasks when selecting the second-best (1 to 2) and third-best (2 to 3) models instead of the top performer.}
  \label{fig:drop}
\end{figure}

\begin{tcolorbox}[finding]
\textbf{Finding 1.} \textit{Agents deployed across devices exhibit strong specialization, suggesting the potential of decentralized cooperation to jointly fulfill diverse tasks.}
\end{tcolorbox}

We first examine the distribution of task-specific strengths across several representative edge-deployable LLM agents. The performance of agents are shown in \autoref{fig:radar}, Agent 1–5 are respectively built on Qwen2.5-7B-Instruct~\cite{yang2025qwen3}, DeepSeek-R1-Distill-Qwen-7B~\cite{guo2025deepseek}, OpenCodeReasoning-Nemotron-7B~\cite{ahmad2025opencodereasoning}, HuatuoGPT-o1-7B~\cite{chen2024huatuogpt}, and Meta-Llama-3-8B-Instruct~\cite{dubey2024llama}. Each agent is equipped with prior knowledge from conversation dataset ShareGPT~\cite{sharegpt}. These agents exhibit systematic differences in their competencies across diverse MMLU tasks.
For example, Agent 4 shows pronounced advantages in medical tasks, reflecting its underlying finetuning on healthcare datasets. Meanwhile, Agent 1 and 2 stand out on mathematical tasks. This complementary pattern indicates that while no single agent dominates universally, each agent provides expertise in a subset of tasks.Consequently, if different users independently deploy specialized agents on their own devices to meet personalized needs, the resulting distributed network could collectively achieve task coverage comparable to having all expertise locally, but without each device bearing the memory cost of hosting agents specialized in all domains.

\begin{tcolorbox}[finding]
\textbf{Finding 2.} \textit{Naively selecting the best agent risks overloading a few nodes, especially in active user environments.}
\end{tcolorbox}

Further examination shows that agent competencies are both specialized and highly imbalanced. \autoref{fig:top} quantifies how often each agent ranks as top-1, top-2, or top-3 across MMLU tasks, revealing that only a few agents dominate the top ranks while most rarely appear among the best. If queries are always routed to the top-performing agent, these few nodes would quickly become overloaded, increasing latency and creating congestion bottlenecks, especially under heavy traffic.
It also suggests that expanding selection to include top-2 or top-3 agents substantially enlarges the candidate pool. By slightly relaxing strict optimality, the system can distribute queries more evenly, alleviating load concentration and improving robustness in large decentralized networks with heterogeneous users. 
This also increases overall system utilization.


\begin{tcolorbox}[finding]
\textbf{Finding 3.} \textit{Routing queries to the second- or third-best agents can effectively alleviate system load while incurring only minimal performance degradation.}
\end{tcolorbox}

\autoref{fig:drop} illustrates the trade-offs when expanding agent selection beyond the top performer. It shows the accuracy drop on MMLU tasks when falling back from the best to the second-best agent (“1 to 2”) and from the second- to the third-best (“2 to 3”). For many tasks, the drop from top-1 to top-2 remains under 5\%, and in several tasks even the top-3 agent provides comparable performance. Although degradation may become steeper beyond top-2, routing some queries to the second- or third-best agents can effectively reduce load while incurring only minor performance loss.
This finding highlights that modest relaxation of strict optimality can improve system efficiency without severely compromising quality. Allocating queries among near-optimal agents helps prevent overload on highly specialized models, whereas excessive relaxation may still harm task accuracy. So, an effective routing strategy should leverage this trade-off to achieve both high performance and balanced system utilization.

These findings motivate our design of PPAI of scalable query-agent pair scoring to handle heterogeneous agents at scale, and congestion-aware multi-agent scheduling to optimize both local and global satisfaction.


\section{System Overview}

Our PPAI system consists of a large network of users connected by P2P communication. With the rapid proliferation of open-source LLMs like Qwen~\cite{yang2025qwen3} and LLaMA~\cite{dubey2024llama}, users can locally deploy these models and adapt them to specific domains like law, medicine, or programming. As illustrated in \autoref{fig:system}, each agent consists of system prompt, tool interfaces and specialized database to achieve personalized capability.

When a user issues a query, it can be served either by the user’s local agent or by another agent in the network that is better suited for the task. To support such flexible and effective collaboration, our system routes each query to the most suitable agent across the network. Building on the limitations identified in prior work (\autoref{relate}) and our findings in \autoref{motivation}, we design the system to meet three key objectives.

• \textbf{Lightweight local deployment:} Each user only needs to maintain a single personalized agent on local device, while still being able to solve diverse tasks through agent interoperability.

• \textbf{Scalable and adaptive coordination:} Our system need to flexibly adapt to dynamic changes in the agent pool, such as agent joining, leaving, or updating, while maintaining scalable task coverage across diverse user requests.

• \textbf{Local and global efficiency balance:} The system aims to balance individual query performance with overall resource utilization by incorporating mechanisms that effectively estimate and account for real-time agent load conditions.

The core design of our system consists of two main components as illustrated in \autoref{fig:system}. First, the query–agent pair scoring module (\autoref{matching}) represents queries and agents in a shared latent space and computes their semantic compatibility. This abstraction enables scalable matching and naturally adapts to the agent pool with churning.
Second, a multi-agent interoperability scheduler module (\autoref{game}) formulates the scheduling as a Bayesian game  complements relevance-based matching with real-time system state by Bayesian game to autonomously select the most suitable peer by balancing semantic relevance against current system congestion. 

The system operates in a continuous loop driven by user queries and agent updates. When a new query arrives at a user, the query–agent pair scoring module first computes the semantic compatibility between the query and all agents, producing a shortlist of candidate agents. The multi-agent interoperability scheduler then refines this decision by incorporating load information and selecting the optimal agent that balances semantic relevance with current system congestion. Through this process, the system achieves effective query–agent matching while maintaining overall load balance.
\begin{figure}[t]
  \centering
  \includegraphics[width=\columnwidth]{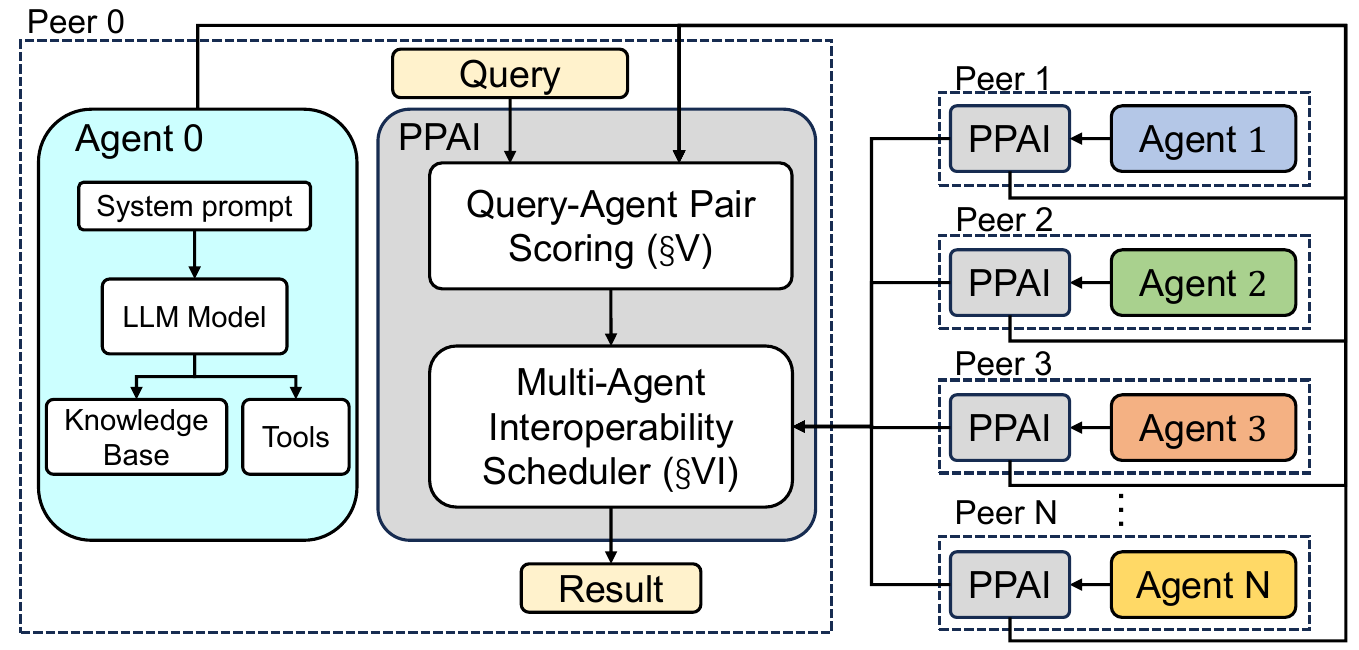}
  \caption{Overview of PPAI system. 
  }
  \label{fig:system}
\end{figure}

\begin{figure*}[t]
  \centering
  \includegraphics[width=0.9\textwidth]{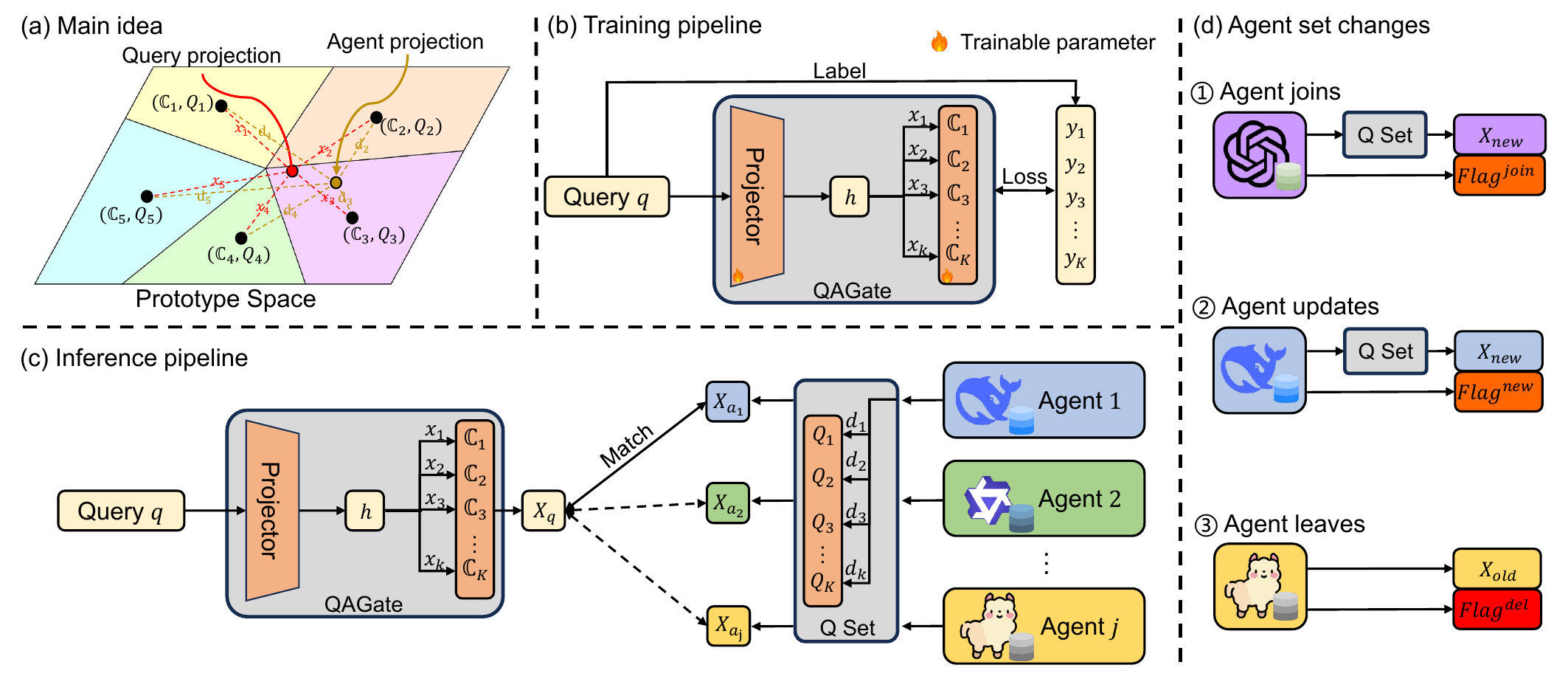}
  \caption{Overview of our prototype-anchored framework for scalable query–agent scoring and matching. (a) Queries and agents are projected into a shared prototype space, where their coordinates relative to pretrained prototypes are computed to obtain relevance scores. (b) During training, the projector learns to map queries to prototype-anchored relevance vectors by labeled supervision. (c) At inference, QAGate produces a relevance vector for each query, which is then matched with the stored capability vectors of candidate agents. (d) The framework supports dynamic agent set changes including join, update, or leave the system by propagating updated capability vectors and associated flags, enabling scalable matching without retraining.}
  \label{fig:matching}
\end{figure*}

\section{Prototype-Anchored Scalable Query-agent Pair Scoring}
\label{matching}
In this section, we introduce our prototype-anchored scalable query–agent pair scoring method. We first discuss the design rationale (\autoref{subsec:idea}) and describe the architecture of the method (\autoref{subsec:proto}). Then we introduce the pair scoring method (\autoref{subsec:matching}) and agent capability updating under churning (\autoref{subsec:gossip}).


\subsection{Design Rationale}
\label{subsec:idea}
Motivated by \textbf{Finding 1} in \autoref{motivation}, we posit that collaboration among complementary agents can achieve high specialization and broad task coverage. As discussed in \autoref{relate}, conventional query–agent matching approaches typically train a centralized router model~\cite{chen2024routerdc, lu2023routing}. By treating matching as classification, these methods implicitly assume a static agent pool. This rigidity renders them ill-suited for P2P environments, where frequent updates would necessitate prohibitively expensive retraining.
Inspired by few-shot classification methods~\cite{snell2017prototypical, sung2018learning, yoon2019tapnet}, which recognize new classes by computing prototype representations from few-shot labeled examples without retraining, we can solve query–agent routing by prototype representations. However, directly applying prototype learning is infeasible as the agent pool is large, heterogeneous, and dynamic. Assigning a dedicated prototype to each agent would create a fragmented and sparse representation space with poor generalization. To overcome this limitation, we construct a prototype-anchored semantic space into which both queries and agents are projected via QAGate in \autoref{fig:matching}. Such design allows us to compute compatibility scores between queries and agents based on prototype-anchored coordinates within a compact latent space, ensuring adaptability to churning.



\subsection{Architecture and initialization}
\label{subsec:proto}
Projecting user queries and agents into a shared latent space requires reliable semantic alignment. To achieve this, we introduce QAGate, illustrated in \autoref{fig:matching}(b). It constructs a prototype-anchored abstraction to capture the inherent semantic diversity of user queries. Each query is first embedded into a dense semantic vector $q \in \mathbb{R}^d$ by a lightweight projector composed of a pretrained encoder followed by MLP. The relevance vector is then the distance between the projection and a set of learned prototypes $\{\mathbb{C}_1, \mathbb{C}_2, \dots, \mathbb{C}_K\}$.

The projector and prototypes are jointly trained by labeled queries collected from a broad range of tasks. These labels can be derived from domain-specific datasets, predefined task categories, or feedback by agents. For instance, a label may correspond to the predefined task category that a query belongs to.
Rather than assigning each query to a single prototype, we explicitly models the nuanced associations between a query and multiple prototypes. Accordingly, we train the projector to produce a relevance vector that approximates a soft target distribution $\hat{p}$ over prototypes. The training objective minimizes the Kullback--Leibler (KL) divergence between $y$ and the softmax-normalized prediction $x =f(q;\theta)$ as:
\begin{equation}
\mathcal{L}_{\mathrm{KL}}(q) 
= \mathrm{KL}\!\bigl(y \,\|\, \mathrm{softmax}(f(q;\theta))\bigr).
\end{equation}

\subsection{Query-agent pair scoring}
\label{subsec:matching}

Inference stage is shown in \autoref{fig:matching}(c), we perform $\ell_2$ normalization on $h(q)$ and each prototype $c_k$, obtaining unit-norm vectors. The relevance score for prototype $c_k$ is computed as:
\begin{equation}
\label{eq:f1}
f(q;\theta)[k] = \frac{h(q)^\top \mathbb{C}_k}{\|h(q)\|_2\,\|\mathbb{C}_k\|_2}, \quad k=1,\dots,K.
\end{equation}

To enforce sparsity and selectivity, we apply a normalization operation to \autoref{eq:f1}. We raise each component to a power $\alpha$ to amplify dominant affinities, retain the top-$p$ fraction of prototypes $\mathrm{Top}_p \subset K$ as:
\begin{equation}
\label{eq:mask}
\tilde{f}(q)[k] =
\begin{cases}
\dfrac{(f(q;\theta)[k])^{\alpha}}{\sum\limits_{j \in \mathrm{Top}_p} (f(q;\theta)[j])^{\alpha}}, & \text{if } k \in \mathrm{Top}_p,\\[2ex]
0, & \text{otherwise}.
\end{cases}
\end{equation}

To represent agent capabilities in the same space, each agent $M_j$ maintains a profile that quantifies its performance (i.e. accuracy) across the $K$ prototypes as a distance $d_k$ to them. Specifically, we evaluate $M_j$ on held-out validation data $\{Q_1, Q_2, \dots, Q_K\}$, yielding a capability vector:
\begin{equation}
p_j = [p(j,1), p(j,2), \dots, p(j,K)] \in \mathbb{R}^K,
\end{equation}
where each component $p(j,k)$ reflects the proficiency of agent $M_j$ on tasks associated with prototype $c_k$. 

When an user $i$ receives a new query $q$, it obtains the masked relevance vector $\tilde{f} = \tilde{f}(q;\theta) \in \mathbb{R}^K$ as described in \autoref{eq:mask}. The suitability of agent $M_j$ for handling this query is then quantified by the cosine similarity between $\tilde{f}$ at user $i$ and $p_j$:
\begin{equation}
\label{eq:score}
s(i,j) 
= \frac{(\tilde{f})^\top p_j}{\|\tilde{f}\|_2\,\|p_j\|_2}.
\end{equation}


\subsection{P2P capability updating}
\label{subsec:gossip}

To enable QAGate in dynamic environment, it is crucial to maintain a consistent and up-to-date global view of agent capabilities. As illustrated in \autoref{fig:matching}(d), three types of agent set changes are handled through gossip-style propagation. When an agent is updated, it recomputes its capability vector using the same query set and broadcasts the update with the timestamp and an update flag. When a new agent joins the system, it computes its initial capability vector and broadcasts it with a join flag and the timestamp. Conversely, when an agent leaves, it broadcasts its identifier with a delete flag. To ensure rapid network-wide consistency, these messages are propagated using a gossip protocol, where each recipient further disseminates to a random subset of peers~\cite{kempe2003gossip}.

\section{Multi-Agent Interoperability Bayesian Game}
\label{game}
To handle unobservable and fluctuating agent loads, we introduce interoperability Bayesian game in this section. It leverages belief distributions to guide routing and flexibly delegate tasks to interchangeable agents.
\subsection{System Model}

\subsubsection{User Set}
We consider a P2P network comprising a set of $N$ users $\mathcal{U}=\{1,2,\dots,N\}$. Each user $j$ locally hosts a personalized agent $M_j$, resulting in heterogeneous specializations across the network. Each user independently decide how to process each incoming query based on locally maintained information.
Each agent $M_j$ has a private type $\theta_j = (\mu_j, \lambda_j)$, where $\mu_j$ denotes its service rate and $\lambda_j$ its query arrival rate. \
The real-time $\theta_j$ is not observable to other agents due to frequent and dynamic changes.

\subsubsection{Belief Distribution}
Since real-time $\theta_j$ is unobservable, each user $i$ maintains a \textit{belief distribution} $b_i(\theta_j)$ over the type space $\Theta_j$. This belief reflects user $i$’s estimation of agent $M_j$’s type. Whenever user $i$ receives state feedback $\omega_j=(\lambda_j, \mu_j)$ from $M_j$, the belief is updated using Bayes’ rule:
\begin{equation}
\label{eq:bnew}
b_i^{\text{new}}(\theta_j) =
\frac{P(\omega_j|\theta_j)b_i(\theta_j)}
{\sum_{\theta'_j\in\Theta_j} P(\omega_j|\theta'_j)b_i(\theta'_j)},
\end{equation}
where $P(\omega_j|\theta_j)$ is the likelihood of observing $\omega_j$ under $\theta_j$.

\subsubsection{Serving Strategy}
When user $i$ receives a query $q_i$, it must decide whether to process it locally or delegate it to another agent. The semantic relevance between $q_i$ and agent $M_j$ is quantified by the prototype-anchored score $s(i,j)$ computed by \autoref{eq:score} in \autoref{matching}. Based on its belief distribution $b_i(\theta_j)$, user $i$ selects an assignment variable $z_i\in\mathcal{J}$, where $\mathcal{J}=\{1,2,\dots,N\}$ is the set of all available agents and $z_i=j$ indicates that $q_i$ is delegated to agent $M_j$.

\subsubsection{System Cost}
To capture congestion effects, we estimate the expected queue load of agent $M_j$ as
\begin{equation}
\label{eq:e}
\mathbb{E}_{\theta_j\sim b_i}[\rho_j] = \min\bigl(1,\,\rho_j + \delta(\lambda_j-\mu_j)\bigr),
\end{equation}
where $\rho_j=\lambda_j/\mu_j$ is the current utilization, $\lambda_j$ and $\mu_j$ are the estimated arrival and service rates under $b_i(\theta_j)$, and $\delta>0$ controls the prediction horizon.

Using \autoref{eq:e}, the \textit{Cost of Delegation (CoD)} for assigning a query from user $i$ to agent $M_j$ is defined as
\begin{equation}
\label{eq:cod}
c(i,j) = \mathbb{E}_{\theta_j\sim b_i}[\rho_j] + t_j^{\text{infer}} + t_{(i,j)}^{\text{trans}},
\end{equation}
where $t_j^{\text{infer}}$ is the expected inference time of $M_j$, and $t_{(i,j)}^{\text{trans}}$ is the expected communication latency between user $i$ and $M_j$.

\subsection{Problem Formulation}
Given the relevance score $s(i,j)$ and the cost of delegation $c(i,j)$, the utility of assigning $q_i$ to agent $M_j$ is
\begin{equation}
\label{eq:utility}
U(i,j) = s(i,j) - \beta\, c(i,j),
\end{equation}
where $\beta>0$ controls the trade-off between semantic alignment and congestion.
The decision-making problem is thus formalized as:
\begin{equation}
\label{eq:z}
z_i = \arg\max_{j\in\mathcal{J}} U(i,j).
\end{equation}

This naturally defines a \textit{Bayesian Game} $\mathcal{G}=\langle \mathcal{U}, \mathcal{J}, \{U(i,j)\}, \{\Theta_j\}, \{b_i\}\rangle$
where each user $i$ is a player, $\theta_j$ is the private type of agent $M_j$, and $b_i(\theta_j)$ is the belief distribution. A \textit{Bayesian Nash Equilibrium} (BNE) is reached when all users adopt strategies that are mutual best responses given their beliefs $b_i(\theta_j)$.

\begin{algorithm}[t]
\caption{Multi-Agent Interoperability at User $i$}
\label{alg:bayesian}
\begin{algorithmic}[1]
\STATE \textbf{Input:} Query $q_i$, belief set $\{b_i(\theta_j)\}$
\STATE Compute prototype relevance $s(i,j)$ for all $j$
\STATE Form candidate set $\mathcal{E}_i = \{ j \mid s(i,j) \ge \theta_s \}$
\STATE Request $(\lambda_j,\mu_j,t_j^{\text{infer}},t_{(i,j)}^{\text{trans}})$ from $j\in\mathcal{E}_i$
\STATE Update beliefs $b_i(\theta_j)$ via Bayes’ rule
\FOR{$j\in\mathcal{E}_i$}
  \STATE Compute CoD $c(i,j)$
  \STATE Compute expected utility $U(i,j)$
\ENDFOR
\STATE Select $z_i = \arg\max_{j\in\mathcal{E}_i} U(i,j)$
\STATE Delegate $q_i$ to $M_{z_i}$ or execute locally if $z_i=i$
\end{algorithmic}
\end{algorithm}

\subsection{Algorithm Design}
We propose algorithm for multi-agent interoperability Bayesian game, as summarized in \autoref{alg:bayesian}. It combines prototype-anchored semantic relevance with Bayesian belief updates to make routing decisions under incomplete information. 
Upon receiving a new query $q_i$, user $i$ obtain query-agent pair scores $s(i,j)$ from \autoref{matching}, forming a candidate set as
\begin{equation}
\mathcal{E}_i = \bigl\{ j \in \mathcal{J} \,\big|\, s(i,j) \geq \theta_s \bigr\},
\end{equation}
where $\theta_s>0$ is a relevance threshold. 

For each candidate agent $M_j\in\mathcal{E}_i$, user $i$ requests state information $(\omega_j,t_j^{\text{infer}},t_{(i,j)}^{\text{trans}})$ to update its belief distribution $b_i(\theta_j)$ via \autoref{eq:bnew}. Using the updated beliefs, it computes the expected queue load (\autoref{eq:e}), derives the CoD $c(i,j)$ as defined in \autoref{eq:cod}, and calculates the expected utility $U(i,j)$ (\autoref{eq:utility}). Finally, user $i$ selects the agent that maximizes $U(i,j)$ as in \autoref{eq:z}.


\subsection{Analysis}

We analyze the proposed multi-agent interoperability Bayesian game by establishing fundamental properties of the best response, potential function, equilibrium existence and uniqueness, efficiency bounds and belief convergence. 

\begin{lemma}
\label{lem:best_response}
Given a fixed candidate set $\mathcal{E}_i$ and belief distribution $b_i(\theta_j)$, user $i$'s best response 
$z_i^\star = \arg\max_{j\in\mathcal{E}_i} U(i,j)$ always exists and is unique.
\end{lemma}

\begin{proof}
For each $j\in\mathcal{E}_i$, $U(i,j)$ is finite since both $s(i,j)$ and $c(i,j)$ are bounded under $b_i(\theta_j)$. Because $\mathcal{E}_i$ is a finite set, there exists at least one $j^\star$ that maximizes $U(i,j)$. A deterministic tie-breaking rule (e.g., selecting the smallest index) guarantees uniqueness of $z_i^\star$.
\end{proof}

\begin{lemma}
\label{lem:monotonicity}
For any user $i$ and agent $M_j$, the utility $U(i,j)$ decreases as the expected queue load $\mathbb{E}_{\theta_j\sim b_i}[\rho_j]$ increases.
\end{lemma}

\begin{proof}
From \autoref{eq:utility} and \autoref{eq:cod}, we have
\begin{equation}
\frac{\partial U(i,j)}{\partial \mathbb{E}_{\theta_j\sim b_i}[\rho_j]}
= -\beta < 0,
\end{equation}
where $\beta>0$. Hence $U(i,j)$ is strictly decreasing in the expected queue load.
\end{proof}



\begin{theorem}[Existence and Uniqueness of Bayesian Nash Equilibrium]
\label{thm:bne_existence_uniqueness}
Define the global potential function
\begin{equation}
\Phi(z) = \sum_{i=1}^N U(i,z_i),
\end{equation}
where $z=(z_1,\dots,z_N)$ is the joint strategy profile. 
The multi-agent interoperability Bayesian game is an exact potential game with potential function $\Phi(z)$ and thus admits at least one Bayesian Nash Equilibrium (BNE). 
Furthermore, if $U(i,j)$ is strictly concave in the mixed-strategy space and users’ beliefs $b_i(\theta_j)$ converge to consistent distributions, the BNE is unique.
\end{theorem}

\begin{proof}
Consider any unilateral deviation of user $i$ from $z_i$ to $z_i'$. The change in the potential function is
\begin{equation}
\Phi(z_i', z_{-i}) - \Phi(z) = U(i,z_i') - U(i,z_i),
\end{equation}
where $z_{-i}$ denotes the strategy profile of all users except $i$. 
Because the change in $\Phi(z)$ equals the change in user $i$’s utility, the game is an exact potential game. 
By the fundamental property of finite exact potential games, the game always admits at least one pure-strategy Nash equilibrium, which corresponds to a BNE in the original Bayesian setting via Harsanyi’s transformation.

If $U(i,j)$ is strictly concave in the mixed-strategy space and users’ beliefs $b_i(\theta_j)$ converge to consistent distributions, then each user has a unique best response (\autoref{lem:best_response}), and $\Phi(z)$ has a unique maximizer. 
Hence, the BNE is unique.
\end{proof}

\begin{lemma}
\label{lemma:convergence}
Under sequential best-response updates, users’ strategies converge to a pure-strategy BNE in finite steps.
\end{lemma}

\begin{proof}
By \autoref{thm:bne_existence_uniqueness}, the game is a finite exact potential game. Each best-response update strictly increases $\Phi(z)$. Because the strategy space is finite, the process must terminate after a finite number of steps at a pure-strategy BNE.
\end{proof}

\begin{theorem}[Bounded Bayesian Price of Anarchy]
\label{thm:bpoa}
For affine CoD in \autoref{eq:cod}, the Bayesian Price of Anarchy (BPoA),
\begin{equation}
\mathrm{BPoA} =
\frac{\displaystyle\max_{z\in\text{BNE}} \mathbb{E}_\theta[\Phi(z)]}
{\displaystyle\min_z \mathbb{E}_\theta[\Phi(z)]},
\end{equation}
is upper bounded by $\frac{5}{3}$.
\end{theorem}

\begin{proof}
Conditioned on each type realization $\theta$, the game reduces to a singleton congestion game with affine cost functions. By the result of Roughgarden and Tardos~\cite{roughgarden2002bad}, the PoA for such games is at most $\frac{5}{3}$. Taking the expectation over $\theta$ preserves this bound, yielding $\mathrm{BPoA}\leq 5/3$.
\end{proof}

\begin{theorem}[Belief Convergence]
\label{thm:belief_convergence}
Assume that feedback $\omega_j$ is repeatedly observed and satisfies the conditional independence assumption 
$P(\omega_j^t|\theta_j)=P(\omega_j|\theta_j)$ for all $t$. 
Then, for any user $i$ and agent $M_j$, the belief distribution $b_i^t(\theta_j)$ updated by Bayes’ rule converges almost surely to the true posterior distribution $P(\theta_j|\omega_j^{1:t})$ as $t\to\infty$.
\end{theorem}

\begin{proof}
At each time step $t$, user $i$ updates its belief about agent $M_j$'s type $\theta_j$ according to Bayes’ rule in \autoref{eq:bnew}.
Because the likelihood $P(\omega_j|\theta_j)$ is correctly specified and the observations $\{\omega_j^t\}_{t=1}^\infty$ are i.i.d.\ conditional on $\theta_j$, the updating sequence satisfies the martingale property with respect to the filtration generated by past observations. For any $t$,
\begin{equation}
\mathbb{E}[\, b_i^{t}(\theta_j) \mid \omega_j^{1:t-1} \,] = b_i^{t-1}(\theta_j).
\end{equation}

By Doob’s martingale convergence theorem~\cite{durrett2019probability}, this bounded martingale converges almost surely to a limit $b_i^\infty(\theta_j)$ as $t\to\infty$. Since the likelihood is correctly specified and the parameter space for $\theta_j$ is assumed to be well-behaved (e.g., finite or satisfying standard regularity conditions), classical Bayesian consistency results imply that the posterior distribution concentrates on the true type. Formally,
\begin{equation}
b_i^t(\theta_j) \to P(\theta_j|\omega_j^{1:t}) \quad \text{a.s. as } t\to\infty,
\end{equation}
and $P(\theta_j|\omega_j^{1:t})$ itself converges to a degenerate distribution at the true $\theta_j$. Hence, $b_i^t(\theta_j)$ converges almost surely to the correct posterior distribution as claimed.

\end{proof}



\section{Discussion}
While our work demonstrates the feasibility and effectiveness in P2P environments, several open challenges remain for future research. One critical direction is to enhance the security and robustness of the system in open P2P settings. In realistic deployments, malicious agents may perform denial-of-service attacks by issuing a large number of meaningless queries, deliberately return incorrect outputs, or attempt to reconstruct private data via model inversion or membership inference attacks through adaptive querying. These vulnerabilities could fundamentally compromise the integrity of collaborative inference. Addressing these threats will require incorporating mechanisms for trust and reputation management, anomaly detection, and rate-limiting policies to mitigate the influence of unreliable or adversarial agents. 

Another important direction is to design effective incentive mechanisms to encourage sustainable and fair collaboration. In practice, the system faces the classic 'free-rider' problem, where rational agents may consume network resources without reciprocal contribution, leading to a tragedy of the commons scenario. Therefore, ensuring that active cooperation becomes the dominant strategy for rational agents is paramount. Future work could explore game-theoretic or blockchain-based incentive schemes that reward agents based on their contributions and reliability, as well as reputation-driven strategies that prioritize cooperative participants. 

\begin{table*}[htbp]
\centering
\caption{Accuracy (\%) and process time (ms) across tasks. 
}
\label{tab:main}
\begin{tabular}{lcc|cc|cc|cc|cc}
\toprule
& \multicolumn{2}{c|}{MMLU}
& \multicolumn{2}{c|}{GSM8K}
& \multicolumn{2}{c|}{MedQA}
& \multicolumn{2}{c|}{ARC-C}
& \multicolumn{2}{c}{AGIEval}\\
& Acc & Time 
& Acc & Time 
& Acc & Time 
& Acc & Time 
& Acc & Time \\
\midrule
\multicolumn{11}{l}{\textit{Single agent}} \\
Agent 1 & 71.24 & 22.70 & 64.54 & 226.80 & 46.94 & 56.00 & 85.32 & 15.47 & 61.05 & 52.44 \\
Agent 2 & 54.43 & 22.70 & 20.91 & 290.63 & 19.62 & 50.15 & 17.75 & 15.47 & 12.75 & 53.85 \\
Agent 3 & 63.02 & 25.16 & 33.33 & 290.63 & 37.21 & 50.83 & 68.94 & 17.65 & 48.38 & 68.07 \\
Agent 4 & 72.91 & 22.70 & 33.64 & 275.94 & 50.86 & 50.38 & 86.69 & 15.88 & 62.13 & 53.85 \\
Agent 5 & 64.01 & 22.70 & 59.39 & 306.50 & 58.08 & 49.70 & 74.74 & 15.88 & 39.18 & 53.61 \\
\midrule
\multicolumn{11}{l}{\textit{Multi-agents}} \\
Voting & 70.87 & 25.16 & 59.39 & 306.50 & 52.12 & 50.83 & 87.03 & 17.65 & 62.06 & 68.07 \\
RouteDC
~\cite{chen2024routerdc}
& 64.98 & 19.56 & 47.58 & 103.24 & 46.94 & 49.48 & 85.67 & 15.13 & 61.36 & 43.81\\
KABB
~\cite{zhang2025kabb}
& 72.29 & 22.40 & 37.27 & 287.25 & 53.53 & 42.67 & 86.35 & 15.96 & 60.20 & 49.64 \\
Ours & 72.94 & 21.05 & 64.55 & 247.31 & 57.46 & 50.37 & 86.01 & 14.69 & 62.62 & 42.12 \\
\bottomrule
\end{tabular}
\end{table*}

\section{Experiment}

\subsection{Implementation}

We build our system on top of vLLM\footnote{https://github.com/vllm-project/vllm}, a fast and easy-to-use library for LLM inference, and libp2p\footnote{https://libp2p.io/}, an open-source networking library for P2P communication. 
For QAGate, we adopt the paraphrase-multilingual-mpnet-base-v2 model~\cite{reimers-2019-sentence-bert} followed by a two-layer MLP with ReLU activation for sentence encoding. To train QAGate for knowledge alignment, we apply AgglomerativeClustering~\cite{cluster} to partition the training samples into 20 clusters, which serve as category labels.
After conducting multiple rounds of hyperparameter search, we retain the best-performing configuration.

\subsection{Experimental Setup}
\textbf{Agent setup.}
We evaluate our system using multiple heterogeneous LLM agents with distinct task specializations. Selected 5 agents each includes a LLM model: 1) Agent 1: Qwen2.5-7B-Instruct~\cite{yang2025qwen3} is finetuned for coding and mathematics abilities, 2) Agent 2: DeepSeek-R1-Distill-Qwen-7B~\cite{guo2025deepseek} is distilled from DeepSeek-R1 based on Qwen, 3) Agent 3: OpenCodeReasoning-Nemotron-7B~\cite{ahmad2025opencodereasoning} is post-trained for reasoning for code generation, 4) Agent 4: HuatuoGPT-o1-7B~\cite{chen2024huatuogpt} is designed for advanced medical reasoning and 5) Agent 5: Meta-Llama-3-8B-Instruct~\cite{dubey2024llama} is optimized for dialogue use cases. Also, each agent is provided prior knowledge from dataset of conversations collected from ShareGPT~\cite{sharegpt}. These agents are based on various structures and fine-tuned for different usage. To evaluate the scalability of our PPAI system, we establish 50, 100, 500, 1000 agents with these 5 models with different set of conversation knowledge.

\textbf{Workloads.}
Requests data are drawn from these datasets: 1) \texttt{MMLU}~\cite{hendrycks2020measuring} is a benchmark that covers 57 general tasks, 2) \texttt{GSM8K}~\cite{cobbe2021training} consists of grade school math word problems, 3) \texttt{MedQA}~\cite{jin2021disease} contains questions from medical exams, \texttt{ARC-C}~\cite{clark2018think} contains grade-school level science questions and \texttt{AGIEval}~\cite{zhong2023agieval} contains 18 tasks from general exams like Gaokao, GRE and SAT. We combine data from \texttt{MMLU}, \texttt{GSM8K}, \texttt{MedQA} and \texttt{ARC-C} and allocate 75\% of examples to train the prototype-gating network and 25\% as user queries, while \texttt{AGIEval} serves as an out-of-distribution (OOD) benchmark with rich categories of tasks. The user queries arrival follow a Poisson distribution with different arrival rates to simulate system performance under heavy demand.

\textbf{Node setup.}
The agents run on Xeon® Gold 6226 CPU and A6000 GPU with 45GB of memory. We consider each user settled cross different regions in a country. Within each region, the delay is 5 ms and bandwidth is 2Gbps.

\textbf{Baselines.}
We compare our system against two types of works, local agent serving and multi-agents serving. There are three baselines of multi-agents serving: 1) Majority voting~\cite{li2024more} is gathering outputs of all agents based on voting, 2) RouterDC~\cite{chen2024routerdc} trains a central router to predict suitable agent, 3) KABB~\cite{zhang2025kabb} construct a knowledge graph and assign queries to pre-defined expert agents. Although these approaches are inherently centralized, we implement their routing strategies within our system for a fair comparison.

\textbf{Metrics.}
We primarily measure \emph{average accuracy} across all queries as the key indicator of task coverage and routing quality. Additionally, we report \emph{average process time}, quantified by the end-to-end processing time per query, incorporating queuing delays and model inference times. We report each result as the mean of 8 independent runs to mitigate randomness. 

\subsection{Results}

\autoref{tab:main} summarizes the comparative performance of our system against both single-agent serving and multi-agent baselines under arrival rate of $\lambda=10$. Majority voting is robust and stable by aggregating outputs from all agents, but suffers from high latency due to issuing redundant queries. RouterDC achieves strong accuracy on certain tasks but exhibits inconsistent results on datasets like GSM8K and MedQA, reflecting the limitations of its centralized router design. KABB delivers performance comparable to ours on many benchmarks but suffers a notable drop on GSM8K, likely because it relies on manually defined task-expert associations that do not generalize well.
In contrast, our approach attains the highest average accuracy across most tasks while substantially reducing inference latency compared to baseline. This demonstrates the effectiveness of PPAI, which enable more precise and consistently reliable query-to-agent assignments.

\begin{figure}[t]
  \centering
  \includegraphics[width=0.48\textwidth]{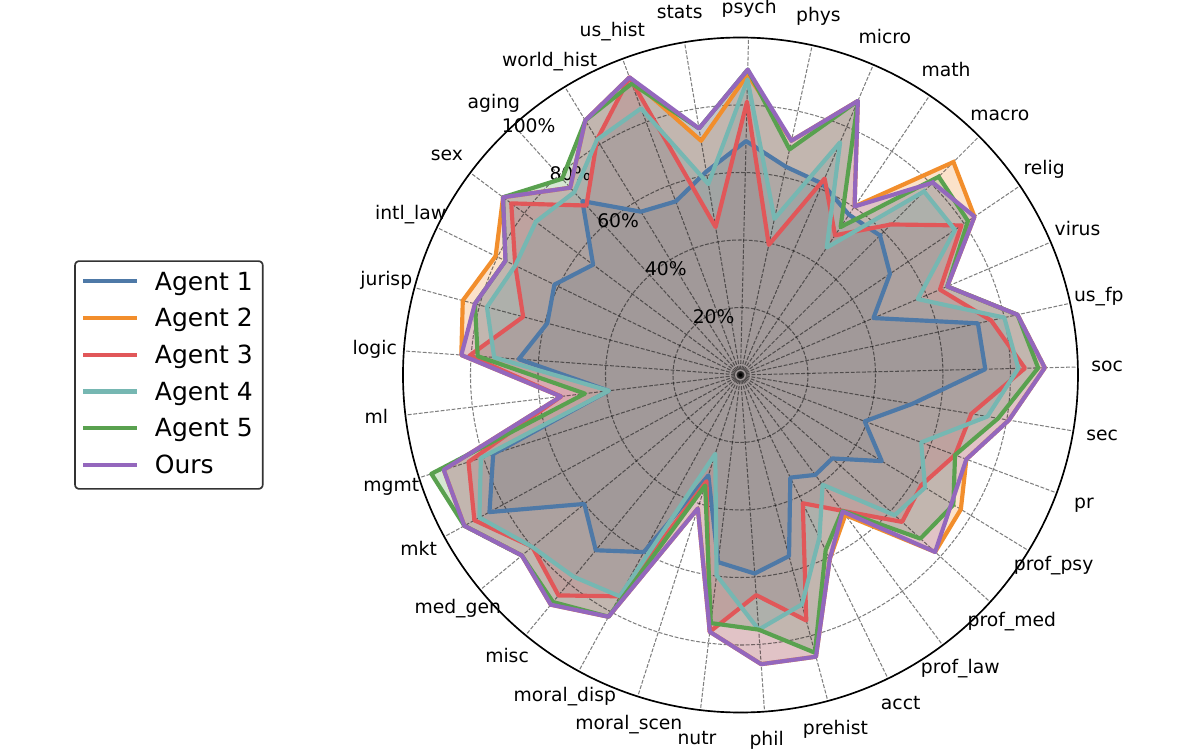}
  \caption{Comparison of candidate models and our method's accuracies across all MMLU tasks.}
  \label{fig:radar_result}
\end{figure}


\autoref{fig:radar_result} visualizes the accuracy distribution across all MMLU task categories, providing a fine-grained perspective on how each approach generalizes over diverse domains. Compared to individual agents, our method maintains a uniformly high level of accuracy across nearly all categories.
This balanced performance pattern highlights that our PPAI method achieves consistently optimal or near-optimal accuracies across a broad spectrum of tasks.

\begin{figure}[t]
  \centering
  \includegraphics[width=0.4\textwidth]{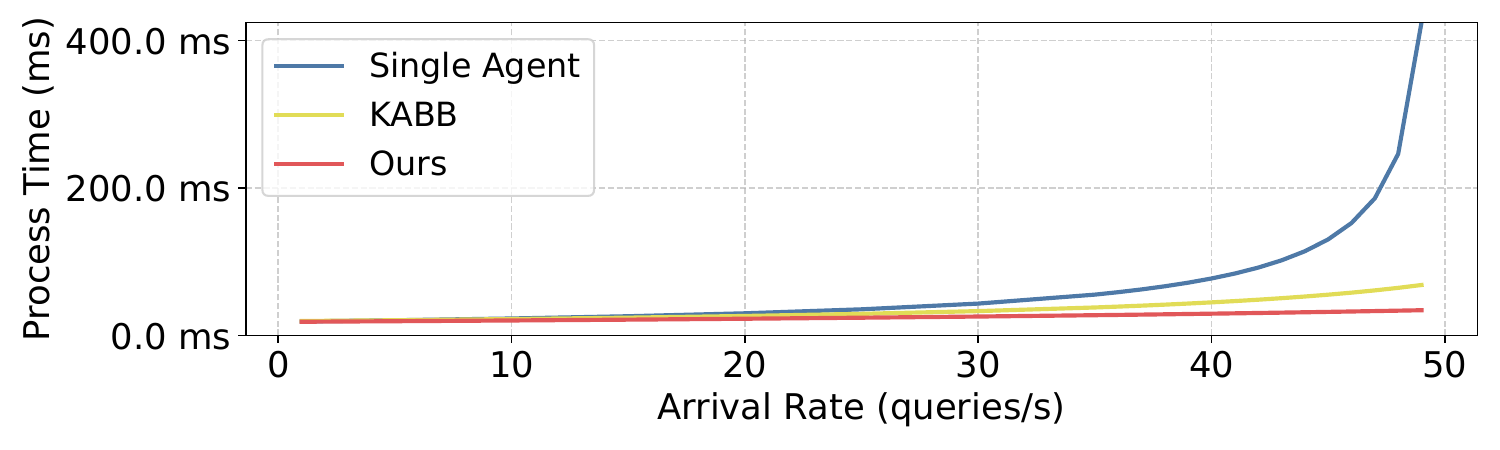}
  \caption{Comparison of processing time under varying arrival rate $\lambda $ on MMLU.}
  \label{fig:arrive}
\end{figure}

While KABB shows comparable performance to our PPAI, we further evaluate their ability on heavier workload of user request. \autoref{fig:arrive} examines system behavior with different arrival rate $\lambda$ on MMLU. We observe that the single-agent configuration rapidly becomes a bottleneck: as $\lambda$ increases beyond 40, the processing time escalates sharply due to queuing delays at the overloaded agent. Our method consistently achieves the lowest overall latency, reflecting its more adaptive routing mechanism under interchangeable agent sets. 

\begin{figure}[t]
  \centering
  \begin{subfigure}{\columnwidth}
    \includegraphics[width=0.93\textwidth]{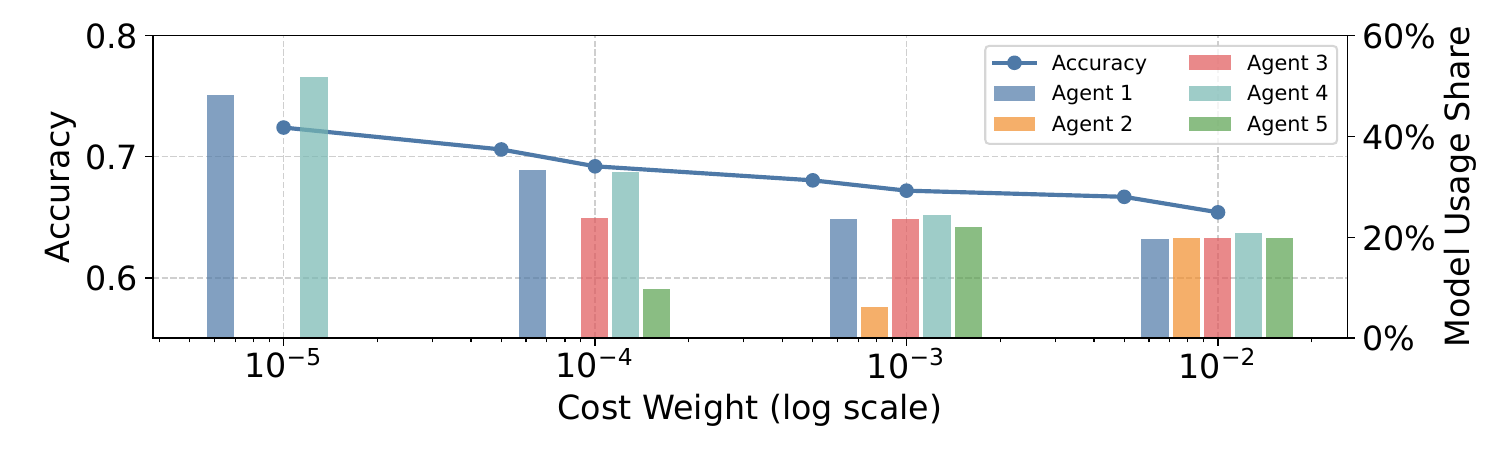}
    \caption{Accuracy.}
    \label{fig:beta-acc-mmlu}
  \end{subfigure}
  \vspace{0.5em}
  \begin{subfigure}{\columnwidth}
    \includegraphics[width=0.93\textwidth]{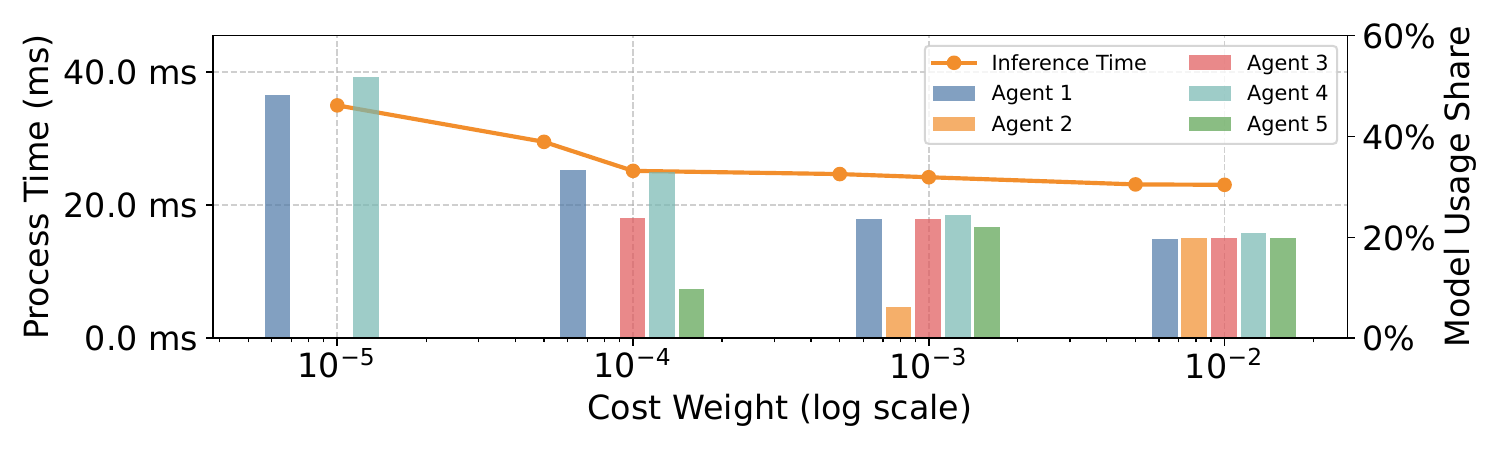}
    \caption{Process time.}
    \label{fig:beta-time-mmlu}
  \end{subfigure}
  \caption{Accuracy and process time on MMLU with varying $\beta$.}
  \label{fig:beta-mmlu}
\end{figure}



\begin{figure}[t]
  \centering
  \includegraphics[width=0.93\columnwidth]{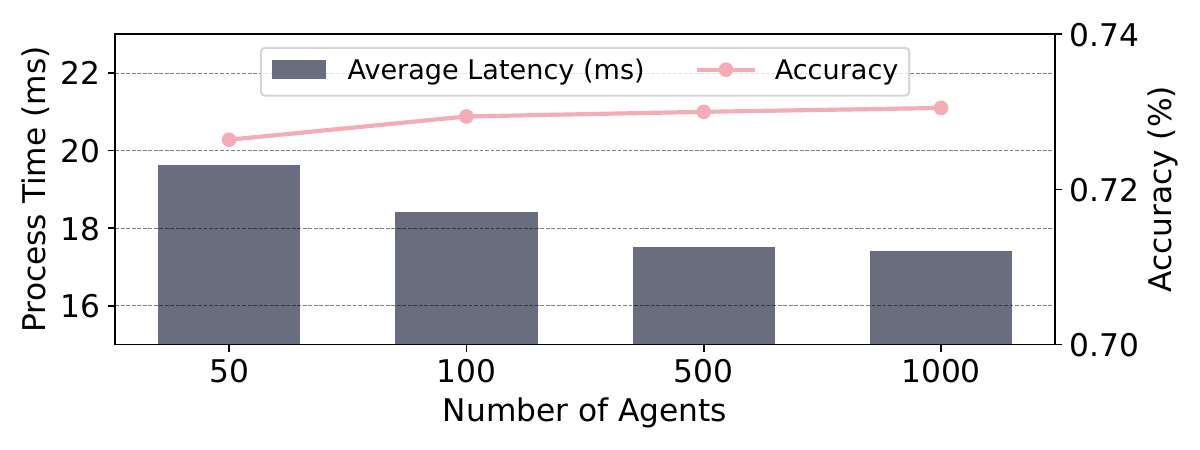}
  \caption{Time and accuracy on MMLU under various scales.}
  \label{fig:scale}
\end{figure}

\autoref{fig:beta-mmlu} investigates how varying the cost hyperparameter $\beta$, which controls sensitivity to queuing delays, affects overall system behavior on \texttt{MMLU}. The arrival rate $\lambda$ is set to 50 to reflect the load balancing ability, The overlaid bar plots reveal how the query distribution across the five agents evolves. As shown in \autoref{fig:beta-acc-mmlu}, increasing $\beta$ from $10^{-5}$ to $10^{-2}$ leads to a steady decline in average accuracy and more even distribution of queries. While \autoref{fig:beta-time-mmlu} shows that higher $\beta$ significantly lowers average processing time. Varying $\beta$ introduces a trade-off between accuracy and congestion, as routing shifts from highly specialized but busy agents to less loaded substitutes. This demonstrates that PPAI can adaptively delegate tasks to alternative agents under high and partially unobservable loads.

\autoref{fig:scale} shows the impact of scaling the number of agents on processing time and accuracy when $\lambda$ is set to 100. As the number of agents increases from 50 to 1000, the average processing time decreases notably due to improved load distribution across a larger pool of agents. Meanwhile, accuracy remains largely stable, indicating that the routing mechanism can effectively utilize interchangeable agents without sacrificing task-specific performance. These results confirm that PPAI is effective with a large scale of agents, achieving lower processing time and stable accuracy.

\section{Conclusion}
In this paper, we present PPAI for collaborative edge intelligence in P2P environments. We address the challenges of diverse, churning and interchangeable agents and limited load visibility.
By introducing a prototype-anchored query–agent pair scoring mechanism, our system enables each user to dynamically match queries to suitable agents within a heterogeneous and churning agent pool. Also, our multi-agent interoperability Bayesian game incorporates belief-based load estimation to guide routing decisions and flexibly delegate tasks to substitute agents. 
Our experiments demonstrate that this approach significantly expands task coverage, improves routing accuracy, and maintains low latency under diverse workloads. Also, our system provides an interface that supports both broad and domain-specific task allocation, while allowing flexible integration of different matching mechanisms according to specific requirements.

\clearpage
\bibliographystyle{IEEEtran}
\bibliography{ref}

@inproceedings{zhang2025kabb,
  title={KABB: Knowledge-Aware Bayesian Bandits for Dynamic Expert Coordination in Multi-Agent Systems},
  author={Zhang, Jusheng and Huang, Zimeng and Fan, Yijia and Liu, Ningyuan and Li, Mingyan and Yang, Zhuojie and Yao, Jiawei and Wang, Jian and Wang, Keze},
  booktitle={ICML},
  year={2025}
}

@inproceedings{chen2024routerdc,
  title={Routerdc: Query-based router by dual contrastive learning for assembling large language models},
  author={Chen, Shuhao and Jiang, Weisen and Lin, Baijiong and Kwok, James and Zhang, Yu},
  booktitle={NeurIPS},
  year={2024}
}

@inproceedings{lu2023routing,
  title={Routing to the expert: Efficient reward-guided ensemble of large language models},
  author={Lu, Keming and Yuan, Hongyi and Lin, Runji and Lin, Junyang and Yuan, Zheng and Zhou, Chang and Zhou, Jingren},
  booktitle={NAACL},
  year={2024}
}

@inproceedings{ong2024routellm,
  title={RouteLLM: Learning to Route LLMs from Preference Data},
  author={Ong, Isaac and Almahairi, Amjad and Wu, Vincent and Chiang, Wei-Lin and Wu, Tianhao and Gonzalez, Joseph E and Kadous, M Waleed and Stoica, Ion},
  booktitle={ICLR},
  year={2025}
}

@inproceedings{jiang2023llm,
  title={Llm-blender: Ensembling large language models with pairwise ranking and generative fusion},
  author={Jiang, Dongfu and Ren, Xiang and Lin, Bill Yuchen},
  booktitle={ACL},
  year={2023}
}

@inproceedings{wang2023fusing,
  title={Fusing models with complementary expertise},
  author={Wang, Hongyi and Polo, Felipe Maia and Sun, Yuekai and Kundu, Souvik and Xing, Eric and Yurochkin, Mikhail},
  booktitle={ICLR},
  year={2024}
}

@inproceedings{yadav2023ties,
  title={Ties-merging: Resolving interference when merging models},
  author={Yadav, Prateek and Tam, Derek and Choshen, Leshem and Raffel, Colin A and Bansal, Mohit},
  booktitle={NeurIPS},
  year={2023}
}

@inproceedings{kempe2003gossip,
  title={Gossip-based computation of aggregate information},
  author={Kempe, David and Dobra, Alin and Gehrke, Johannes},
  booktitle={FOCS},
  year={2003}
}

@inproceedings{huang2019gpipe,
  title={Gpipe: Efficient training of giant neural networks using pipeline parallelism},
  author={Huang, Yanping and Cheng, Youlong and Bapna, Ankur and Firat, Orhan and Chen, Dehao and Chen, Mia and Lee, HyoukJoong and Ngiam, Jiquan and Le, Quoc V and Wu, Yonghui and others},
  booktitle={NeurIPS},
  year={2019}
}

@inproceedings{chen2018tvm,
  title={$\{$TVM$\}$: An automated $\{$End-to-End$\}$ optimizing compiler for deep learning},
  author={Chen, Tianqi and Moreau, Thierry and Jiang, Ziheng and Zheng, Lianmin and Yan, Eddie and Shen, Haichen and Cowan, Meghan and Wang, Leyuan and Hu, Yuwei and Ceze, Luis and others},
  booktitle={OSDI 18},
  pages={578--594},
  year={2018}
}

@inproceedings{frantar2023optq,
  title={OPTQ: Accurate post-training quantization for generative pre-trained transformers},
  author={Frantar, Elias and Ashkboos, Saleh and Hoefler, Torsten and Alistarh, Dan-Adrian},
  booktitle={ICLR},
  year={2023}
}

@inproceedings{gu2023minillm,
  title={MiniLLM: Knowledge distillation of large language models},
  author={Gu, Yuxian and Dong, Li and Wei, Furu and Huang, Minlie},
  booktitle={ICLR},
  year={2024}
}

@inproceedings{hu2022lora,
  title={Lora: Low-rank adaptation of large language models.},
  author={Hu, Edward J and Shen, Yelong and Wallis, Phillip and Allen-Zhu, Zeyuan and Li, Yuanzhi and Wang, Shean and Wang, Lu and Chen, Weizhu and others},
  booktitle={ICLR},
  year={2022}
}

@inproceedings{lester2021power,
  title={The power of scale for parameter-efficient prompt tuning},
  author={Lester, Brian and Al-Rfou, Rami and Constant, Noah},
  booktitle={EMNLP},
  year={2021}
}

@inproceedings{lewis2020retrieval,
  title={Retrieval-augmented generation for knowledge-intensive nlp tasks},
  author={Lewis, Patrick and Perez, Ethan and Piktus, Aleksandra and Petroni, Fabio and Karpukhin, Vladimir and Goyal, Naman and K{\"u}ttler, Heinrich and Lewis, Mike and Yih, Wen-tau and Rockt{\"a}schel, Tim and others},
  booktitle={NeurIPS},
  year={2020}
}

@inproceedings{ouyang2022training,
  title={Training language models to follow instructions with human feedback},
  author={Ouyang, Long and Wu, Jeffrey and Jiang, Xu and Almeida, Diogo and Wainwright, Carroll and Mishkin, Pamela and Zhang, Chong and Agarwal, Sandhini and Slama, Katarina and Ray, Alex and others},
  booktitle={NeurIPS},
  year={2022}
}

@inproceedings{asai2023self,
  title={Self-rag: Learning to retrieve, generate, and critique through self-reflection},
  author={Asai, Akari and Wu, Zeqiu and Wang, Yizhong and Sil, Avirup and Hajishirzi, Hannaneh},
  booktitle={ICLR},
  year={2023}
}

@inproceedings{cohen2003incentives,
  title={Incentives build robustness in BitTorrent},
  author={Cohen, Bram},
  booktitle={P2P Econ},
  volume={6},
  pages={68--72},
  year={2003}
}

@inproceedings{maymounkov2002kademlia,
  title={Kademlia: A peer-to-peer information system based on the xor metric},
  author={Maymounkov, Petar and Mazieres, David},
  booktitle={IPTPS},
  pages={53--65},
  year={2002}
}

@inproceedings{jelasity2005gossip,
  title={Gossip-based aggregation in large dynamic networks},
  author={Jelasity, M{\'a}rk and Montresor, Alberto and Babaoglu, Ozalp},
  booktitle={TOCS},
  volume={23},
  number={3},
  pages={219--252},
  year={2005}
}

@article{belcak2025small,
  title={Small Language Models are the Future of Agentic AI},
  author={Belcak, Peter and Heinrich, Greg and Diao, Shizhe and Fu, Yonggan and Dong, Xin and Muralidharan, Saurav and Lin, Yingyan Celine and Molchanov, Pavlo},
  journal={arXiv preprint arXiv:2506.02153},
  year={2025}
}

@article{roughgarden2002bad,
  title={How bad is selfish routing?},
  author={Roughgarden, Tim and Tardos, {\'E}va},
  journal={Journal of the ACM},
  volume={49},
  number={2},
  pages={236--259},
  year={2002}
}

@article{guo2025deepseek,
  title={Deepseek-r1: Incentivizing reasoning capability in llms via reinforcement learning},
  author={Guo, Daya and Yang, Dejian and Zhang, Haowei and Song, Junxiao and Zhang, Ruoyu and Xu, Runxin and Zhu, Qihao and Ma, Shirong and Wang, Peiyi and Bi, Xiao and others},
  journal={arXiv preprint arXiv:2501.12948},
  year={2025}
}

@article{ahmad2025opencodereasoning,
  title={Opencodereasoning: Advancing data distillation for competitive coding},
  author={Ahmad, Wasi Uddin and Narenthiran, Sean and Majumdar, Somshubra and Ficek, Aleksander and Jain, Siddhartha and Huang, Jocelyn and Noroozi, Vahid and Ginsburg, Boris},
  journal={arXiv preprint arXiv:2504.01943},
  year={2025}
}

@article{chen2024huatuogpt,
  title={Huatuogpt-o1, towards medical complex reasoning with llms},
  author={Chen, Junying and Cai, Zhenyang and Ji, Ke and Wang, Xidong and Liu, Wanlong and Wang, Rongsheng and Hou, Jianye and Wang, Benyou},
  journal={arXiv preprint arXiv:2412.18925},
  year={2024}
}

@article{dubey2024llama,
  title={The llama 3 herd of models},
  author={Dubey, Abhimanyu and Jauhri, Abhinav and Pandey, Abhinav and Kadian, Abhishek and Al-Dahle, Ahmad and Letman, Aiesha and Mathur, Akhil and Schelten, Alan and Yang, Amy and Fan, Angela and others},
  url={https://ai.meta.com/research/publications/the-llama-3-herd-of-models/},
  year={2024}
}

@article{yang2025qwen3,
  title={Qwen2.5: A Party of Foundation Models!},
  author={Yang, An and Li, Anfeng and Yang, Baosong and Zhang, Beichen and Hui, Binyuan and Zheng, Bo and Yu, Bowen and Gao, Chang and Huang, Chengen and Lv, Chenxu and others},
  url={https://qwenlm.github.io/blog/qwen2.5/},
  year={2025}
}

@inproceedings{hendrycks2020measuring,
  title={Measuring massive multitask language understanding},
  author={Hendrycks, Dan and Burns, Collin and Basart, Steven and Zou, Andy and Mazeika, Mantas and Song, Dawn and Steinhardt, Jacob},
  booktitle={ICLR},
  year={2021}
}

@article{cobbe2021training,
  title={Training verifiers to solve math word problems},
  author={Cobbe, Karl and Kosaraju, Vineet and Bavarian, Mohammad and Chen, Mark and Jun, Heewoo and Kaiser, Lukasz and Plappert, Matthias and Tworek, Jerry and Hilton, Jacob and Nakano, Reiichiro and others},
  journal={arXiv preprint arXiv:2110.14168},
  year={2021}
}

@article{jin2021disease,
  title={What disease does this patient have? a large-scale open domain question answering dataset from medical exams},
  author={Jin, Di and Pan, Eileen and Oufattole, Nassim and Weng, Wei-Hung and Fang, Hanyi and Szolovits, Peter},
  journal={Applied Sciences},
  volume={11},
  number={14},
  pages={6421},
  year={2021}
}

@article{clark2018think,
  title={Think you have solved question answering? try arc, the ai2 reasoning challenge},
  author={Clark, Peter and Cowhey, Isaac and Etzioni, Oren and Khot, Tushar and Sabharwal, Ashish and Schoenick, Carissa and Tafjord, Oyvind},
  journal={arXiv preprint arXiv:1803.05457},
  year={2018}
}

@inproceedings{zhong2023agieval,
  title={Agieval: A human-centric benchmark for evaluating foundation models},
  author={Zhong, Wanjun and Cui, Ruixiang and Guo, Yiduo and Liang, Yaobo and Lu, Shuai and Wang, Yanlin and Saied, Amin and Chen, Weizhu and Duan, Nan},
  booktitle={NAACL},
  year={2024}
}

@inproceedings{li2024more,
  title={More agents is all you need},
  author={Li, Junyou and Zhang, Qin and Yu, Yangbin and Fu, Qiang and Ye, Deheng},
  booktitle={TMLR},
  year={2024}
}

@article{cluster,
  title={AgglomerativeClustering},
  author={Scikit-learn},
  url={https://scikit-learn.org/stable/modules/generated/sklearn.cluster.AgglomerativeClustering.html},
  year={2024}
}

@inproceedings{reimers-2019-sentence-bert,
    title = "Sentence-BERT: Sentence Embeddings using Siamese BERT-Networks",
    author = "Reimers, Nils and Gurevych, Iryna",
    booktitle = "EMNLP",
    year = "2019"
}

@inproceedings{snell2017prototypical,
  title={Prototypical networks for few-shot learning},
  author={Snell, Jake and Swersky, Kevin and Zemel, Richard},
  booktitle={NeurIPS},
  year={2017}
}

@inproceedings{sung2018learning,
  title={Learning to compare: Relation network for few-shot learning},
  author={Sung, Flood and Yang, Yongxin and Zhang, Li and Xiang, Tao and Torr, Philip HS and Hospedales, Timothy M},
  booktitle={CVPR},
  year={2018}
}

@inproceedings{yoon2019tapnet,
  title={Tapnet: Neural network augmented with task-adaptive projection for few-shot learning},
  author={Yoon, Sung Whan and Seo, Jun and Moon, Jaekyun},
  booktitle={ICML},
  year={2019}
}

@book{durrett2019probability,
  title={Probability: Theory and Examples},
  author={Durrett, Rick},
  edition={5},
  year={2019}
}

@inproceedings{lu2025demystifying,
  title={Demystifying Small Language Models for Edge Deployment},
  author={Lu, Zhenyan and Li, Xiang and Cai, Dongqi and Yi, Rongjie and Liu, Fangming and Liu, Wei and Luan, Jian and Zhang, Xiwen and Lane, Nicholas D and Xu, Mengwei},
  booktitle={ACL},
  year={2025}
}

@article{sharegpt,
  title={ShareGPT},
  author={OpenAI},
  url={https://huggingface.co/datasets/RyokoAI/ShareGPT52K},
  year={2024}
}

@article{wang2025d,
  title={D${}^2$MoE: Dual Routing and Dynamic Scheduling for Efficient On-Device MoE-based LLM Serving},
  author={Wang, Haodong and Zhou, Qihua and Hong, Zicong and Guo, Song},
  journal={MobiCom},
  year={2025}
}

@inproceedings{fang2025klotski,
  title={Klotski: Efficient Mixture-of-Expert Inference via Expert-Aware Multi-Batch Pipeline},
  author={Fang, Zhiyuan and Huang, Yuegui and Hong, Zicong and Lyu, Yufeng and Chen, Wuhui and Yu, Yue and Yu, Fan and Zheng, Zibin},
  booktitle={ASPLOS},
  year={2025}
}

@inproceedings{liu2025mell,
  title={Mell: Memory-Efficient Large Language Model Serving via Multi-GPU KV Cache Management},
  author={Liu, Qianli and Hong, Zicong and Li, Peng and Chen, Fahao and Guo, Song},
  booktitle={INFOCOM},
  year={2025}
}

@inproceedings{deng2023mind2web,
  title={Mind2web: Towards a generalist agent for the web},
  author={Deng, Xiang and Gu, Yu and Zheng, Boyuan and Chen, Shijie and Stevens, Sam and Wang, Boshi and Sun, Huan and Su, Yu},
  booktitle={NeurIPS},
  year={2023}
}

@inproceedings{yang2024swe,
  title={Swe-agent: Agent-computer interfaces enable automated software engineering},
  author={Yang, John and Jimenez, Carlos E and Wettig, Alexander and Lieret, Kilian and Yao, Shunyu and Narasimhan, Karthik and Press, Ofir},
  booktitle={NeurIPS},
  year={2024}
}

@inproceedings{hong2023metagpt,
  title={MetaGPT: Meta programming for a multi-agent collaborative framework},
  author={Hong, Sirui and Zhuge, Mingchen and Chen, Jonathan and Zheng, Xiawu and Cheng, Yuheng and Wang, Jinlin and Zhang, Ceyao and Wang, Zili and Yau, Steven Ka Shing and Lin, Zijuan and others},
  booktitle={ICLR},
  year={2023}
}

@inproceedings{wang2024mobile,
  title={Mobile-agent-v2: Mobile device operation assistant with effective navigation via multi-agent collaboration},
  author={Wang, Junyang and Xu, Haiyang and Jia, Haitao and Zhang, Xi and Yan, Ming and Shen, Weizhou and Zhang, Ji and Huang, Fei and Sang, Jitao},
  booktitle={NeurIPS},
  year={2024}
}

@inproceedings{liu2025mobilesteward,
  title={MobileSteward: Integrating Multiple App-Oriented Agents with Self-Evolution to Automate Cross-App Instructions},
  author={Liu, Yuxuan and Sun, Hongda and Liu, Wei and Luan, Jian and Du, Bo and Yan, Rui},
  booktitle={KDD},
  year={2025}
}

@article{fan2025minirag,
  title={Minirag: Towards extremely simple retrieval-augmented generation},
  author={Fan, Tianyu and Wang, Jingyuan and Ren, Xubin and Huang, Chao},
  journal={arXiv preprint arXiv:2501.06713},
  year={2025}
}

\end{document}